\pdfoutput=1
\documentclass[]{article}

\usepackage{amsmath,amsfonts,amssymb}
\usepackage{amsthm}
\usepackage{manfnt}
\usepackage{graphicx}
\usepackage{bbold}
\usepackage[numbers,sort&compress]{}
\usepackage[colorlinks=true, allcolors=blue]{hyperref}
\usepackage[margin=1in]{geometry}

\usepackage{mathtools}
\mathtoolsset{showonlyrefs}

\newtheorem{corollary}{Corollary}

\newtheorem{theorem}{Theorem}
\newtheorem{lemma}{Lemma}

\def\indep{\perp\!\!\!\perp}

\def\R{\mathbb{R}}
\def\E{\mathbb{E}}
\def\P{\mathbb{P}}
\def\cI{\mathcal{I}}
\def\cJ{\mathcal{J}}

\newcommand{\eqd}{\stackrel{\textnormal{d}}{=}}

\newcommand{\cube}{\mbox{\mancube}}
\newcommand{\mathbbm}[1]{\text{\usefont{U}{bbm}{m}{n}#1}}
\def\ind#1{\mathbbm{1}\left\{#1\right\}}

\title{Wavefront Randomization Improves Deconvolution}

\author{Amit Kohli\textsuperscript{*}, Anastasios N. Angelopoulos\thanks{equal contribution} , Laura Waller}

\date{University of California, Berkeley \\~\\
\today}

\begin{document}
\maketitle

\begin{abstract}
   The performance of an imaging system is limited by optical aberrations, which cause blurriness in the resulting image.
   Digital correction techniques, such as deconvolution, have limited ability to correct the blur, since some spatial frequencies in the scene are not measured adequately (i.e., `zeros' of the system transfer function). 
   We prove that  the addition of a random mask to an imaging system removes its dependence
   on aberrations, reducing the likelihood of zeros in the transfer function and consequently decreasing the sensitivity to
   noise during deconvolution.
   In simulation, we show that this strategy improves image quality over a range of aberration types, aberration strengths, and signal-to-noise ratios.
\end{abstract}


\section{INTRODUCTION}
\label{sec:intro}  

Aberrations describe the deviations of an imaging system from ideal, diffraction-limited imaging. 
Even well-designed optics have inherent aberrations; they are often the limiting factor in optical space-bandwidth product.  Correcting aberrations usually involves complex sequences of optical elements---like those in a microscope objective---to achieve diffraction-limited imaging across the target field-of-view (FoV). Alternatively, adaptive optical systems use a programmable phase modulator in a closed-loop to dynamically correct aberrated wavefronts in real-time. Both of these hardware-based solutions are generally expensive and bulky. 
For a simpler and less expensive alternative, many users turn to computational post-processing---for example, deconvolution, whereby images captured with poorly-corrected optics are digitally processed to remove aberration effects. Deconvolution requires knowing or measuring the system point spread function (PSF), then implementing an image reconstruction algorithm to deconvolve it from the captured image.
However, deconvolution is limited in use as it often fails in the case of low-quality and/or noisy images.

This manuscript explores a new computational imaging approach to correcting aberrations via a simple and inexpensive hardware modification combined with standard deconvolution. Given an aberrated imaging system, we show that \emph{wavefront randomization} (e.g. by inserting a random phase mask in the pupil plane) can result in improved deconvolution (see Fig.~\ref{fig:teaser}). It may be surprising to think that adding a random scattering element to the system could improve image quality. Indeed, the captured images from the randomized system will initially look worse than the original aberrated images.
However, randomization better encodes the scene's spatial frequency information---i.e., it improves its transfer function, resulting in a higher-quality image after deconvolution. Wavefront randomization can be implemented by a simple phase mask or diffuser in the pupil plane, whose PSF can be experimentally measured for use in deconvolution.  

Our approach is enabled by a new discovery: adding a uniformly random phase mask to an imaging system makes it invariant to its initial aberrations. 
When a system is aberrated, it induces structured wavefront distortions that often cause zeros in the system's transfer function. 
However, when the wavefront distortions are uniformly randomized, they are no longer correlated with the original aberrations and thus lose their structure, resulting in a transfer function with no zeros. 
Consequently, this random but improved transfer function makes deconvolution more tolerant to noise, which would have normally overcome the signal null frequencies.

Aberration correction methods using phase masks in conjunction with computation have been explored before and can be divided into nonrandom designs~\cite{Frieden:68, Cathey:84, harvey2011control} which optimize for a specific aberration type, and random designs which trade off optimality for robustness to unknown aberrations. These random methods have performed aberration correction in the context of extending the depth of field~\cite{Zalevsky_2008}, correcting sample-induced aberrations~\cite{meitav2016point}, heuristic transfer function design~\cite{STOSSEL1995156}, stellar interferometry with low quality optics~\cite{DAINTY1973129}, and sparse aperture imaging~\cite{Miller:07}.
This manuscript provides rigorous statistical analyses for randomized imaging under \emph{arbitrary, unknown} aberrations, accompanied by comprehensive simulations of transfer function distributions/means and deconvolutions against aberration type and noise.
The theorems we prove are new, and give \emph{exact} analytical expressions for the random transfer functions under consideration.

\begin{figure}[t]
   \centering
   \includegraphics[width=0.8\textwidth]{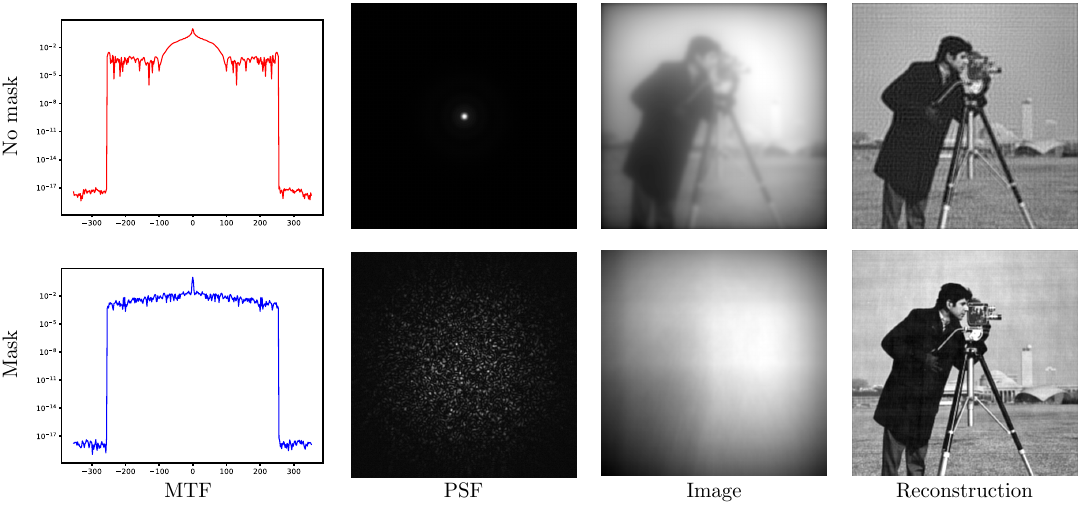}
   \caption{Simulation of a spherically aberrated imaging system with and without wavefront randomization. 
   With no randomization (top row) the system has an MTF with severe nulls and a large blob-like point spread function (PSF). 
   The image is a blurry, noisy version of the scene and the deconvolved image has noise-induced patterned artifacts.
   With a random mask (bottom row), the wavefront is randomized, causing the modulation transfer function (MTF) to become flatter, with no nulls. The corresponding point spread function (PSF) is a speckle pattern with small features. 
   The image is random, but the deconvolved image is much closer to the ground truth. Noise is white, additive Gaussian.}
   \label{fig:teaser}
\end{figure}

\section{Background}
\label{sec:background}

For the purpose of convenience, although imaging systems are generally 2-dimensional, the theory in this manuscript is done in a 1-dimensional discrete-time setting.
All optical fields will be described as discrete, periodic complex-valued sequences with period $N$. 
This mathematical setup is very common in the development of imaging algorithms~\cite{paxman1992joint}.
All sequences will henceforth be defined on $\{0, \ldots, N-1\}$, with the understanding that they can be naturally extended to $\mathbb{Z}$ by periodicity.

The pupil function
\begin{equation}
   \label{eq:pupil-def}
   P_n = A_ne^{i\phi_n}
\end {equation}
   is composed of two real-valued sequences: the transmittance of the system aperture $A_n$, and the deviation $\phi_n$ from an ideal wavefront.
   These $\phi_n$ are the undesirable, unknown aberrations that we wish to correct.
   For the remainder of this manuscript, we will set $A_n = 0$ for $n = \lfloor \frac{N}{2}\rfloor, \dots, N-1$ and $A_n = 1$ otherwise.
   The modulation transfer function (MTF) $H_n$ is the real-valued magnitude of the discrete autocorrelation of the pupil function, 
\begin{equation}
    \label{eq:mtf-def}
    H_n = \bigg|\frac{1}{||P_n||^2}\sum_{m = 0} ^ {N-1} P_mP^*_{m-n}\bigg| = \bigg|\frac{1}{||P_n||^2}\sum_{m = 0} ^ {N-1} A_me^{i\phi_m}A_{m-n}e^{-i\phi_{m-n}}\bigg|.
\end{equation}
The MTF describes the amount of each spatially frequency in the scene captured by the system.
Clearly, the aberrations $\phi_n$ affect the value of the MTF.
In fact, an application of the Cauchy-Schwarz inequality to~\eqref{eq:mtf-def} gives a pointwise upper bound on the MTF:
\begin{equation}
    H_n \leq 
    \bigg|\frac{1}{||P_n||^2} \sum_{m = 0} ^ {N-1} A_mA_{m-n}\bigg| = \bigg| 1 -  \frac{|n|}{\lfloor N/2\rfloor} \bigg|.
    \label{eq:mtf-cs}
\end{equation}
Equality occurs when $\phi_n = 0$, or, in other words, \emph{aberrations can only worsen a diffraction-limited system}.~\cite{Frieden:68}

Our main tool will be the use of a phase mask to change the aberration profile.
If $W_n$ is a sequence representing the phase profile of a phase mask, the pupil becomes
\begin{equation}
   \label{eq:pupil-masked}
   \tilde{P}_n = A_ne^{i(\phi_n + W_n)},
\end {equation}
and the resulting MTF follows from autocorrelation as per~\eqref{eq:mtf-def} (see Appendix~\ref{appendix:lemmas}).
This paper is about choosing $W_n$ strategically in order to remove the dependence of the transfer function on the aberrations---an outcome we call \emph{aberration invariance}.

The benefit of aberration invariance is to improve the effectiveness of digital post-processing.
Although theoretically exact recovery is possible for any positive transfer function, in reality, any zeros or near-zeros in the MTF are highly susceptible to noise.
These noisy frequencies are boosted by the inverse filter, leading to noticeable artifacts in the deconvolution.
The top row of Fig.~\ref{fig:teaser} shows an example of a spherical aberration that pushes the MTF below the noise floor at several null frequencies, leading to a systematic, patterned artifacts corresponding to those frequencies.
Achieving aberration invariance will allow us to provably avoid this type of deconvolution artifact.

\section{Theory}
\label{sec:theory}
We now describe how wavefront randomization via random masks provides aberration invariance. 

\subsection{Random Masks}
\label{subsec:random-masks}
Our main discovery is that aberration independence can be provably achieved by wavefront randomization, \emph{without any knowledge about the aberrations whatsoever}. 
Herein we do so by inserting a random phase mask into the pupil plane with phase profile $W_n$, which is a real-valued random variable whose distribution we can design.

The theoretical results in this section provide characterizations of the transformed MTF under two different models for $W_n$: in the first and simplest model, $W_n \sim \mathrm{Unif(0,2\pi)}$. 
In the second, $W_n \sim \pi \mathrm{Bern}(0.5)$.
In both cases, the theory will show aberration independence arising from inserting the mask.
Figure~\ref{fig:theory} shows accompanying simulations of the MTFs and supports the theoretical claims.

\subsubsection{Uniform mask}
In this section, we consider the case where $W_n \sim \mathrm{Unif(0, 2\pi)}$.
It will be immediately clear that the resulting pupil function is entirely independent of aberrations and has a known and exact distribution.

\begin{theorem}[Aberration invariance: uniform mask]
    \label{thm:uniform-mtf}
    Consider a masked pupil function $\tilde{P}_n$ as in~\eqref{eq:pupil-masked} with arbitrary aberrations $\phi_n$ and $W_n \overset{i.i.d.}{\sim} \mathrm{Unif(0, 2\pi)}$. Then,
    \begin{equation}
        \tilde{P}_n \eqd A_ne^{iW_n}
    \end{equation}
    and
    \begin{equation}
        H_n \eqd \frac{1}{\lfloor N/2 \rfloor}\left| \mathbbm{1}_{C(N,n)}^\top e^{i\Delta_n(W)} \right|, 
    \end{equation}
    where $C(N,n) = \lfloor N/2 \rfloor - n$ and the $\Delta_n : \R^N \to \R^{C(N,n)}$ function computes the vector $\Delta_n(w) = (w_n - w_0, \ldots, w_j - w_{j-n}, \ldots, w_{N-1} - w_{N-n-1})$.
\end{theorem}

Intuitively, by uniformly randomly shifting the phase at each point in the pupil wavefront, the mask makes it so that the wavefront itself is uniformly random, regardless of what the initial aberrations were. 
The proof of this theorem relies on the unique fact that uniform random phasors are invariant to constant shifts (this fact is also critical to the study of random speckle patterns~\cite{goodman2007speckle}).
Figure~\ref{fig:theory} confirms this theorem and also shows that the resulting random MTF concentrates around its mean, avoiding null frequencies with essentially probability one.


\begin{figure}[t]
   \centering
   \includegraphics[width=0.8\textwidth]{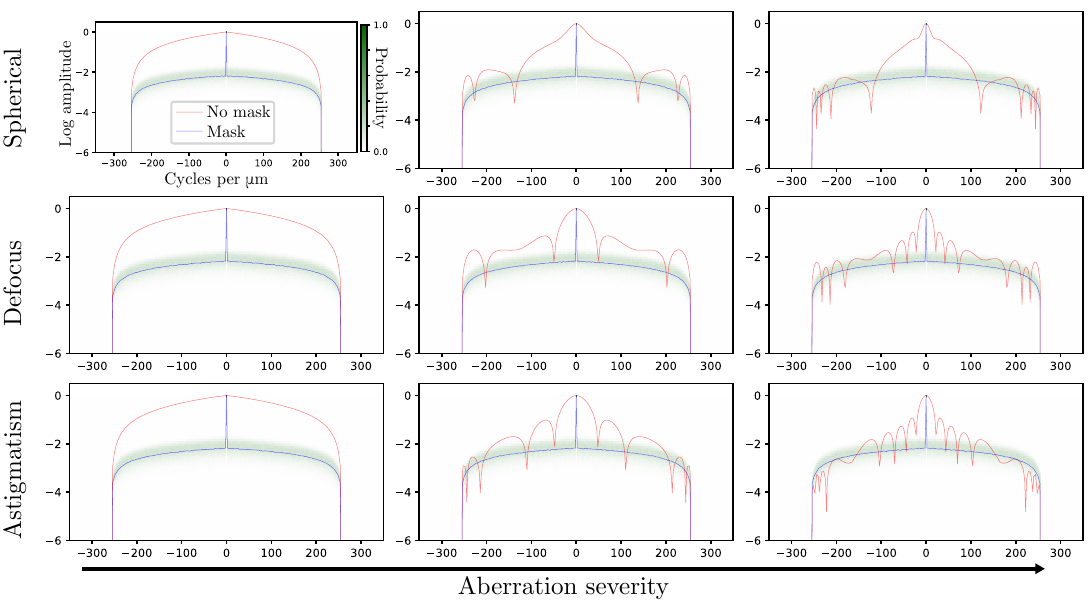}
   \caption{Simulation of MTFs with and without a uniform random mask. Each row represents a different aberration type, and each column represents a different aberration strength. 
   Within each individual plot is the MTF of the system with no mask (red), the empirical distribution of MTFs from many draws of a uniform mask (green), and the average MTF of those draws (blue).
   As expected by Theoreom~\ref{thm:uniform-mtf}, the MTF distribution and average from uniform masks do not change with aberration type or strength, whereas the MTF without a mask does so drastically.
   Also note how the MTF distribution is concentrated around the average, signifying that the MTF is reliably null-free.}
   \label{fig:theory}
\end{figure}

\subsubsection{Binary mask}
\label{subsec:binary-mask}

In this section, we consider the case where $W_n \sim \pi \mathrm{Bern}(0.5)$, i.e. $W_n$ is an (appropriately scaled) Bernoulli random variable with probability $p=0.5$.
This binary mask can be easily fabricated or represented with an adaptive element like a deformable mirror. An exact characterization of the MTF is possible, but it is aberration-dependent; nonetheless, a lower-bound on the expectation can be derived that is independent of the aberrations.
\begin{theorem}[Approximate aberration invariance: binary mask]
    \label{thm:binary-mtf}
    Consider a pupil function $\tilde{P}_n$ as in~\eqref{eq:pupil-masked} with arbitrary aberrations $\phi$ and $W_n \overset{i.i.d.}{\sim} \pi \mathrm{Bern}(0.5)$. Then,
    \begin{equation}
        \bar{H}_n \eqd \frac{\sqrt{C(N,n)}}{\lfloor N/2 \rfloor} \sqrt{1 + \frac{2a^\top U}{C(N,n)}},
    \end{equation}
    where $U$ is uniformly distributed on $\cube_n$: the $C(N,n)$-dimensional hypercube with vertices $\{-1,1\}$ along each dimension and $a$ is the $C(N,n)$-dimensional vector
    \begin{equation}
        \label{eq:a-aberration-term}
        a_{jk} = \left( \left| \cos(\phi_j - \phi_{j-n} - \phi_k + \phi_{k-n} ) \right| \right), \forall j \in \{n, \ldots, N-1\}, k \in \{j, \ldots, N-1\}.
    \end{equation}
    Furthermore, the expectation of the MTF is
    \begin{equation}
        \mathbb{E}[\bar{H}_n] = D(N,n) \sum\limits_{u \in \cube_n}\sqrt{1 + \frac{2a^\top u}{C(N,n)}} \geq   D(N,n) \sum\limits_{u \in \cube_n}\sqrt{\left(1 + \frac{2\mathbbm{1}^\top U}{C(N,n)}\right)_+},
    \end{equation}
    where $(\cdot)_+ = x\ind{x \geq 0}$ and $D(N,n) = \frac{1}{\lfloor N/2 \rfloor2^{\sqrt{C(N,n)}(C(N,n)-1)}}$.
\end{theorem}

Above, the MTF is expressed as a function of the aberrations through the vector $a$ and the random variable $U$.
The dependence is mild and the MTF is approximately independent of the aberrations (see Appendix~\ref{appendix:binary-mask}).
Furthermore, the expectation of the MTF is lower-bounded by a quantity that is aberration invariant.

\begin{figure}[t]
    \centering
    \includegraphics[width=\textwidth]{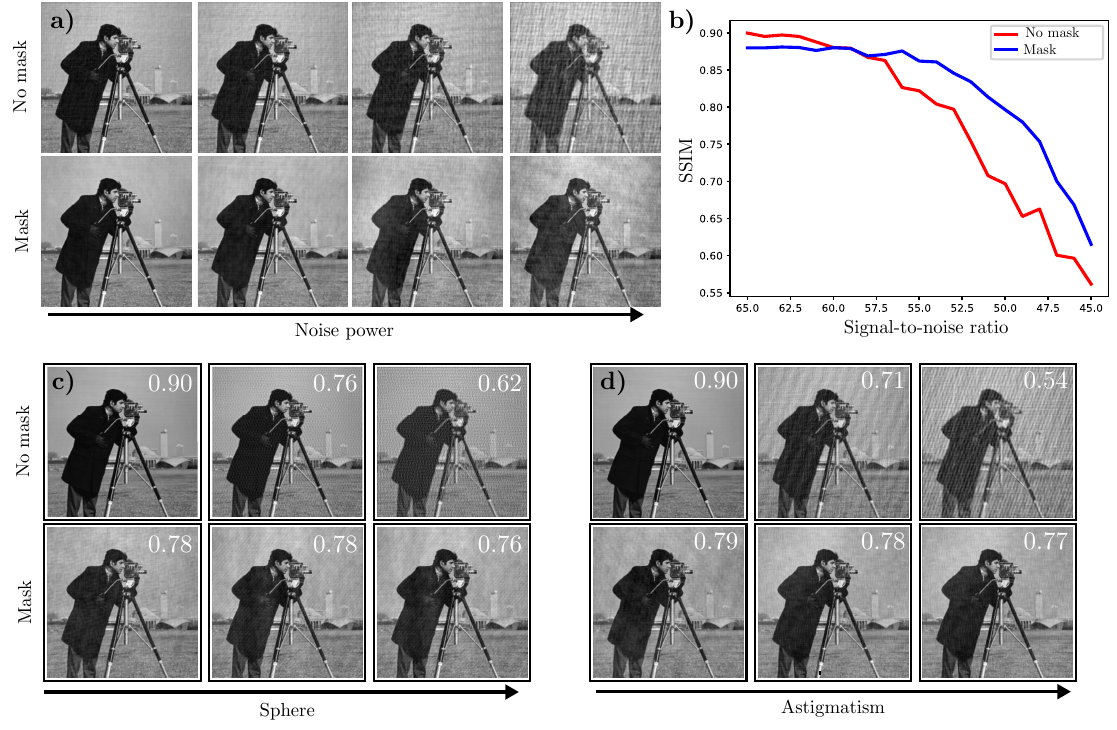}
    \caption{Simulation of image reconstructions and their SSIM scores for different noise levels and aberration strengths/types. \textbf{a)} Deconvolutions with (bottom) and without (top) a uniform random mask for increasing noise power. The masked case degrades more slowly. \textbf{b)} A plot of SSIM scores on the deconvolution from a) against SNR; the masked deconvolution degrades more gradually. \textbf{c)} Deconvolutions with and without a mask for increasing levels of spherical aberration and \textbf{d)} astigmatism. Deconvolution with the mask is less sensitive to noise for large aberrations. SSIM scores are shown on each image reconstruction.}
    \label{fig:recons}
 \end{figure}

\section{Imaging Simulations}
\label{sec:simulations}
The theoretical results in the previous section show that wavefront randomization can remove the MTF's dependence on aberrations.
Moreover, the resulting random MTFs are concentrated and rarely have nulls, making them amenable to post-processing.
It remains to perform deconvolutions to determine if this strategy actually improves image quality in the presence of noise.

To that end, we simulate imaging with and without a random mask for a variety of different conditions: aberration types, aberration strengths, and signal-to-noise ratios.
The simulation is done by first generating a pupil function with an aberration profile determined by Seidel aberration coefficients; the size of the coefficients correspond to aberration strength.
This aberrated pupil is then masked with a random mask---drawn \emph{only once} at the beginning of the experiment---generated by sampling from a uniform distribution.
Then, by Fourier transform of the pupil, we obtain the system point spread function (PSF) and convolve it with the cameraman image to obtain the measurement.
Finally, we add Gaussian noise to the measurement and PSF, and use a Wiener filter to deconvolve the noisy image with the noisy PSF.

The simulation is repeated with and without the mask, using sphere and astigmatism as the base aberrations. The aberration strength and noise levels are varied from low to high. The results are displayed in Fig.~\ref{fig:recons} along with SSIM scores.
The main takeaway of this experiment is that the reconstruction without the mask is heavily dependent on the aberration-noise combination, whereas the reconstruction with the mask is solely dependent on the noise, regardless of the aberration level.
Thus, for a fixed noise level, the reconstruction with the mask is nearly identical and has similar SSIM scores for all aberration levels and types. 
This is in stark contrast to the reconstruction without the mask, which degrades severely with aberration level.

\section{Discussion}
The primary contribution of this manuscript is the discovery that wavefront randomization can remove the dependence of an imaging system on its aberrations.
Specifically, by using a random pupil mask, the transfer function of an aberrated system is transformed into a random transfer function whose distribution is independent of the aberrations.
Moreover, this random transfer function has desirable properties for use in deconvolution such as being concentrated around its mean and rarely having zeros.
Within a certain noise regime, these random transfer functions allow for better deconvolution, even with severe aberrations.

The next logical step in this inquiry is to specify a practical imaging regime in which wavefront randomization is beneficial.
Real-life experiments are needed to capture a variety of variables beyond the scope of simple simulations and determine whether this method has utility.

On the theoretical side, there are still many open questions about the properties of the random transfer functions and whether there are superior mask distributions for particular imaging conditions. Further, providing analytic high-probability lower bounds on the MTF, would be of interest; since the exact MTF distribution is known, it may also be possible to design better recovery algorithms by leveraging these statistics in a nonparametric maximum likelihood model.
The resampling and combination of multiple random masks in order to improve image reconstruction is also an interesting topic.
Additionally, assumption of shift-invariance played a major role in both theory and simulations since it provides the simple relationships between the pupil, MTF, and PSF; it is true however that many highly-aberrated systems are actually shift-varying. 
But, by dividing a shift-varying system into isoplanatic patches---a common existing strategy---the theory can be easily extended to shift-varying systems as well.

It is also worth noting that the theory in this manuscript is done in 1-dimensional discrete-time, but the results can be extended to 2-dimensional and even continuous time in a straightforward manner. 
Finally, we believe the randomization of other optical elements, and the analysis of their effect on the distribution of the resulting reconstruction, is an important avenue for future research under the paradigm of wavefront randomization.

\section*{ACKNOWLEDGMENTS}       
A.K. was funded by the Berkeley Fellowship for Graduate Study. 
A.N.A. was supported by the Berkeley Fellowship for Graduate Study and the National Science Foundation Graduate Research Fellowship Program under Grant No. DGE 1752814. Any opinions, findings, and conclusions or recommendations expressed in this material are those of the author(s) and do not necessarily reflect the views of the National Science Foundation. 
This material is based upon work supported by the Air Force Office of Scientific Research under award number FA9550-22-1-0521. 
This publication has been made possible in part by CZI grant DAF2021-225666 and grant DOI https://doi.org/10.37921/192752jrgbn from the Chan Zuckerberg Initiative DAF, an advised fund of Silicon Valley Community Foundation (funder DOI 10.13039/100014989).

\bibliographystyle{unsrt}
\bibliography{report.bib}

\begin{thebibliography}{10}

\bibitem{Frieden:68}
B.~Roy Frieden.
\newblock How well can a lens system transmit entropy?
\newblock {\em J. Opt. Soc. Am.}, 58(8):1105--1112, 1968.

\bibitem{Cathey:84}
W.~Thomas Cathey, B.~Roy Frieden, William~T. Rhodes, and Craig~K. Rushforth.
\newblock Image gathering and processing for enhanced resolution.
\newblock {\em J. Opt. Soc. Am. A}, 1(3):241--250, 1984.

\bibitem{harvey2011control}
Andrew~R Harvey, Gonzalo~D Muyo, and Tom Vettenburg.
\newblock Control of optical aberrations with coded apertures.
\newblock In {\em Unconventional Imaging, Wavefront Sensing, and Adaptive Coded
  Aperture Imaging and Non-Imaging Sensor Systems}, volume 8165, pages
  310--317. SPIE, 2011.

\bibitem{Zalevsky_2008}
Zeev Zalevsky and Alex Zlotnik.
\newblock Axially and transversally super-resolved imaging and ranging with
  random aperture coding.
\newblock {\em Journal of Optics A: Pure and Applied Optics}, 10(6):064014,
  2008.

\bibitem{meitav2016point}
Nizan Meitav, Erez~N Ribak, and Shy Shoham.
\newblock Point spread function estimation from projected speckle illumination.
\newblock {\em Light: Science \& Applications}, 5(3):e16048--e16048, 2016.

\bibitem{STOSSEL1995156}
Bryan~J. Stossel and Nicholas George.
\newblock Multiple-point impulse responses: controlled blurring and recovery.
\newblock {\em Optics Communications}, 121(4):156--165, 1995.

\bibitem{DAINTY1973129}
J.C. Dainty.
\newblock Diffraction-limited imaging of stellar objects using telescopes of
  low optical quality.
\newblock {\em Optics Communications}, 7(2):129--134, 1973.

\bibitem{Miller:07}
Nicholas~J. Miller, Matthew~P. Dierking, and Bradley~D. Duncan.
\newblock Optical sparse aperture imaging.
\newblock {\em Appl. Opt.}, 46(23):5933--5943, 2007.

\bibitem{paxman1992joint}
Richard~G Paxman, Timothy~J Schulz, and James~R Fienup.
\newblock Joint estimation of object and aberrations by using phase diversity.
\newblock {\em JOSA A}, 9(7):1072--1085, 1992.

\bibitem{goodman2007speckle}
Joseph~W Goodman.
\newblock {\em Speckle phenomena in optics: theory and applications}.
\newblock Roberts and Company Publishers, 2007.

\end{thebibliography}
\appendix
\section{Binary mask}
\label{appendix:binary-mask}
This section will expand on the theoretical development initialized in Sec.~\ref{subsec:binary-mask}. Accompanying simulations of binary mask MTFs are displayed in Fig.~\ref{fig:theory-binary} As a brief review, insertion of a mask, $W_n \overset{i.i.d.}{\sim} \pi \mathrm{Bern}(0.5)$, into the Fourier plane of an imaging system yields a MTF
\begin{equation}
        H_n \eqd \frac{\sqrt{C(N,n)}}{\lfloor N/2 \rfloor} \sqrt{1 + \frac{2a^\top U}{C(N,n)}},
\end{equation}
where $C(N,n) = \lfloor N/2 \rfloor - n$ and $a$ contains the aberrations as per Eq.~\eqref{eq:a-aberration-term}. Though this quantity is not completely aberration independent, it has a weak dependence on aberrations. This is an empirical observation illustrated in Fig.~\ref{fig:theory-binary}. There is also theoretical motivation for such a claim: the second moment of $\bar{H}_n$ is aberration independent. To see this, first consider the 2nd moment of 
$\bar{H}_n$ for a generally Bernoulli parameter (not necessarily 0.5).
\begin{figure}[ht]
   \centering
   \includegraphics[width=0.8\textwidth]{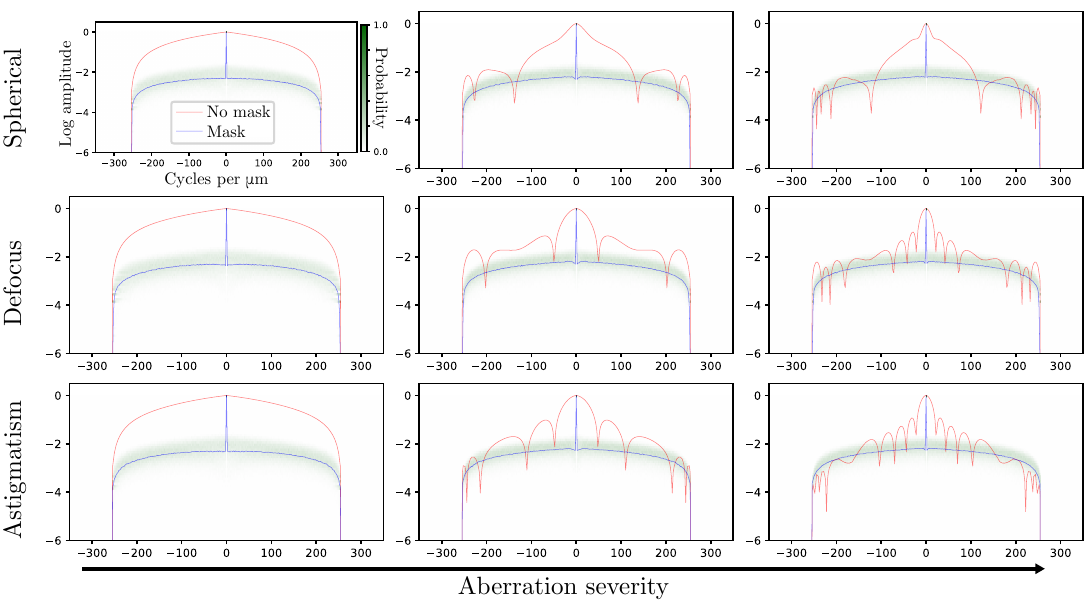}
   \caption{Simulation of MTFs with and without a binary random mask. Each row represents a different aberration type, and each column represents a different aberration strength. 
   Within each individual plot is the MTF of the system with no mask (red), the empirical distribution of MTFs from many draws of a binary mask, and the average MTF of those draws (blue). The MTF distribution and average from binary masks do change with aberration type and strength, but very little. 
   Moreover, the binary mask MTF distribution is concentrated around the average, signifying that the MTF is reliably null-free.}
   \label{fig:theory-binary}
\end{figure}

\begin{theorem}
    Consider the pupil function $\tilde{P_j}$, which has been masked by a Bernoulli-$p$ phase mask. Then for $n = 0, \dots, \lfloor N/2\rfloor - 1$ its autocorrelation $H_n$ has a 2nd moment

    \begin{align*}
        \mathbb{E}[H_n^2] &=  \mathbbm{1}\{n=0\} + \mathbbm{1}\{n>0\} \Biggr[\left(\frac{1}{\lfloor N/2\rfloor}-\frac{n}{\lfloor N/2\rfloor^2}\right) \\ &+ \left(\frac{(1-2p)^2}{\lfloor N/2\rfloor^2}\right) \left(\sum_{j=n}^{\lfloor N/2\rfloor - n - 1} e^{i(2\phi_j - \phi_{j-n} - \phi_{j+n})} +  \sum_{j=2n}^{\lfloor N/2\rfloor - 1}e^{i(\phi_j - 2\phi_{j-n}  + \phi_{j-2n}})\right) 
         \\ &+ \left(\frac{(1-2p)^4}{\lfloor N/2\rfloor^2}\right) \biggr(\sum_{j=n}^{\lfloor N/2\rfloor - 1}\sum_{\substack{k \neq j \\ k \neq j \pm n}} e^{ i (\phi_j - \phi_{j-n} - \phi_k + \phi_{k-n})}\biggr) \Biggr].
    \end{align*}
    \label{thm:binary-squared-mtf}
\end{theorem}
Now, when $p = 0.5$, we see that all aberration dependent terms vanish.
\begin{corollary}[Aberration invariance: squared MTF of binary mask]
    \label{thm:binary-squared-mtf-0.5}
    Consider the setting of Theorem~\ref{thm:binary-squared-mtf} with $p=0.5$. Then,
    \begin{equation}
        \E\left[H_n^2\right] = \frac{C(N,n)}{\lfloor N/2 \rfloor}.
    \end{equation}
\end{corollary}
The 2nd moment of $H_n$ or the average squared MTF is a constant and independent of any aberrations, signifying that the concentration properties of the MTF about its expectation are favorable.



\section{Proofs of Main Results}
\label{proofs-main}
Below are proofs of the theorems shown in the main manuscript, refer to Appendix~\ref{appendix:lemmas} for supporting lemmas.
\begin{proof}[Proof of Theorem~\ref{thm:uniform-mtf}]

    The first statement directly follows from Lemma~\ref{lem:translation-invariance-uniform}.
    By Lemma~\ref{lem:pupil-to-MTF},  
    \begin{equation}
        H_n = \frac{1}{\lfloor N/2 \rfloor} \left| \sum\limits_{j=n}^{N-1} e^{i(\phi_j - \phi_{j-n} + W_j - W_{j-n})}\right|
    \end{equation}
    where $C(N,n) = \lfloor N/2 \rfloor - n$.
    Two applications of Lemma~\ref{lem:translation-invariance-uniform} gives
    \begin{equation}
        H_n \eqd \frac{1}{\lfloor N/2 \rfloor} \left| \sum\limits_{j=n}^{N-1} e^{i(W_j - W_{j-n})}\right|,
    \end{equation}
    which is succinctly written in the theorem statement.
    This completes the proof.
\end{proof}

\begin{proof}[Proof of Theorem~\ref{thm:binary-mtf}]
We only consider the MTF for $n= 1, \dots, \lfloor \frac{N}{2} \rfloor - 1$ since it is symmetric. We can write the unnormalized MTF as 
\begin{equation}
        \label{eq:intermediate-MTF-rademacher}
        \bar{H}_n = \left | \sum_{j=0}^{N-1} \tilde{P}_j\tilde{P}^\ast_{j-n} \right | =  \left | \sum_{j=0}^{N-1} A_n^2 e^{i(\phi_j - \phi_{j-n})}e^{i\pi B_j}e^{-i\pi B_{j-n}} \right | = \left | \sum_{j=n}^{\lfloor N/2\rfloor - 1} e^{i(\phi_j - \phi_{j-n})} R_j R_{j-n} \right |,
\end{equation}
where now the Bernoulli $B_j$ random variables have probability $0.5$ and by Lemma~\ref{lem:rademacher-phase} the corresponding $R_j$ are Rademacher random variables.

Thus, we can write the MTF as
\begin{align}
        \bar{H}_n &=  \sqrt{ \left| \sum_{j=n}^{\lfloor N/2\rfloor - 1} e^{i(\phi_j - \phi_{j-n})} R_j R_{j-n} \right|^2} \\
        &= \sqrt{(\lfloor N/2 \rfloor - n) + \sum_{j=n}^{\lfloor N/2 \rfloor - 1}\sum_{k \in \{n, \ldots, j-1, j+1, \lfloor N/2 \rfloor -1\}} R_jR_{j-n}R_kR_{k-n}e^{i(\phi_j - \phi_{j-n} - \phi_k + \phi_{k-n})}} \\
        &= \sqrt{C(N,n)}\left(\sqrt{1 + \frac{1}{C(N,n)}Z}\right),
\end{align}

where $C(N,n)=\lfloor N/2 \rfloor - n$ and $Z = \sum_{j=n}^{\lfloor N/2\rfloor - 1}\sum_{k \in \{n, \ldots, j-1, j+1, \lfloor N/2 \rfloor -1\}} R_jR_{j-n}R_kR_{k-n}e^{i(\phi_j - \phi_{j-n} - \phi_k + \phi_{k-n})}$.

Let the summands of $Z$ be $Z_{j,k} = R_jR_{j-n}R_kR_{k-n}e^{i\Delta^{\phi}_{j,k}}$ where $\Delta^{\phi}_{j,k} = \phi_j - \phi_{j-n} - \phi_k + \phi_{k-n}$, and define the index set 
\begin{equation}
    \cJ = \{ (j,k) : j \in \{n, \ldots, \lfloor N/2 \rfloor \} \} \text{ and } k \in \{n, \ldots, \lfloor N/2 \rfloor \}\setminus \{j\}.
\end{equation}
Consider the distribution of the summand with index $(j,k) \in \cJ$, i.e., $z_{j,k}$.
Because $j \neq k$, $R_j \indep R_k$, so $R_jR_k$ is equal in distribution to a Rademacher random variable.
There are two further cases.
\paragraph{ (1) $\mathbf{j \neq k-n}$ \textbf{and} $\mathbf{k \neq j-n}$} In this case, $R_jR_kR_{j-n}R_{k-n} \sim \mathrm{Rad}$ because all are clearly independent.
\paragraph{ (2) $\mathbf{j = k-n}$ \textbf{or} $\mathbf{k = j-n}$} In the first sub-case,  case, $R_j R_{k-n} = 1$ (deterministically), so $R_j R_k R_{j-n} R_{k-n} = R_k R_{j-n} \sim \mathrm{Rad}$ by independence. The second sub-case is similar.

The conclusion of the above cases is that $Z_{j,k} \eqd R_{j,k} e^{i\Delta^{\phi}_{j,k}}$ for some Rademacher random variable $R_{j,k}$.
What is the dependency structure of these Rademachers?
The following cases illuminate the question:
\paragraph{ \textbf{(1)} $\mathbf{(j',k') = (k,j)}$}. We have that $R_{j,k} = R_{k,j}$; they are deterministically equal.
\paragraph{ \textbf{(2)} $\mathbf{j=k-n}$ \textbf{and not (1)} } It can be verified exhaustively that $R_{j,k} \mid R_{j',k'} \sim \mathrm{Rad}$, which implies independence.

Knowing this dependence structure, we can re-express $Z$ as
\begin{equation}
    Z \eqd \sum_{j=n}^{\lfloor N/2\rfloor - 1}\sum_{k > j} R_{j,k} \left( e^{i\Delta^{\phi}_{j,k}} + e^{-i\Delta^{\phi}_{j,k}}\right) = 2\sum_{(j,k) \in \cI_>} R_{j,k} \cos \left(\Delta^{\phi}_{j,k}\right).
\end{equation}
Above, the $>$ in the index of the sums and the splitting of the exponential are due to the fact that $R_{j,k} = R_{k,j}$ and $e^{i\Delta^{\phi}_{j,k}} = e^{-i\Delta^{\phi}_{k,j}}$.
The index set $\cI_>$ is all pairs appearing in the earlier sums; $|\cI_>| = C(N,n)(C(N,n) - 1)$.
It is now worth noting that $Z$ has gone from being complex-valued to real-valued, and that the $R_{j,k} \in \cI_{>}$ are independent.

First, consider the case $n = \lfloor N/2\rfloor-1$. Here, $Z \overset{a.s.}{=} 0$, which implies $\E[ \bar{H}_n ] = \sqrt{C(N,n)} = 1$.

Now consider a vector $U$ uniformly distributed over the vertices of the hypercube
\begin{equation}
    \cube = \pm \mathbbm{1}_{|\cI_>|}.
\end{equation}
We can write, by Lemma~\ref{lem:prod-rad-unif},
\begin{equation}
    Z \eqd 2a^{\top} U,
\end{equation}
where $a$ is the $|\cI_>|$-length vector $\left(\left|\cos(\Delta^{\phi}_{j,k})\right|\right)_{(j,k) \in \cI_>}$.
Thus, the full distribution of the (normalized) MTF is $H_n \eqd \frac{\sqrt{C(N,n)}}{\lfloor N/2 \rfloor}\sqrt{1+\frac{2a^\top U}{C(N,n)}}$, which proves the first theorem statement.

Now we are ready to calculate $\E\left[\sqrt{1+\frac{Z}{C(N,n)}}\right]$ manually.
The expectation is equal to
\begin{equation}
    \sum\limits_{u \in \cube} \sqrt{1 + \frac{2 a^\top u}{C(N,n)}}\P(U=u) =  \frac{1}{2^{|\cI_>|}} \sum_{u \in \cube} \sqrt{1+\frac{2a^\top u}{C(N,n)}}.
\end{equation}
As a sum of concave functions, this is concave (and non-constant) in $a$.
Thus, its minimum occurs at a corner point; but which one?

We can see that, for all $q \in [|\cI_{>}|]$, letting $\cI_{>}^{-q} = \cI_{>}\setminus\{q\}$,
\begin{equation}
    \sum_{u \in \cube} \sqrt{1+\frac{2a^\top u}{C(N,n)}} = \sum_{u \in \cube_{\cI_{>}^{-q}}} \left(\sqrt{1+\frac{2a_{\cI_{>}^{-q}}^\top u}{C(N,n)} + \frac{2a_q}{C(N,n)}} + \sqrt{1+\frac{2a_{\cI_{>}^{-q}}^\top u}{C(N,n)} - \frac{2a_q}{C(N,n)}}\right).
\end{equation}
Each summand in the above display is of the form
\begin{equation}
    \sqrt{1+\gamma+\frac{2a_q}{C(N,n)}} + \sqrt{1+\gamma-\frac{2a_q}{C(N,n)}}.
\end{equation}
But the minimizing value of $a_q$ over the domain $[0,1]$, uniformly over all feasible values of $\gamma$, is $a_q=1$.
Similarly, the maximizer is $a_q=0$.
(This can be verified by taking derivatives, or alternatively, by plotting this function.)
Thus,
\begin{align}
    &\arg\min_{a_q} \sum_{u \in \cube_{\cI_{>}^{-q}}} \left(\sqrt{1+\frac{2a_{\cI_{>}^{-q}}^\top u}{C(N,n)} + \frac{2a_q}{C(N,n)}} + \sqrt{1+\frac{2a_{\cI_{>}^{-q}}^\top u}{C(N,n)} - \frac{2a_q}{C(N,n)}}\right) \\
    = &\arg\min_{a_q} \sqrt{1+\gamma+\frac{2a_q}{C(N,n)}} + \sqrt{1+\gamma-\frac{2a_q}{C(N,n)}} = 1.
\end{align}
This directly implies that $a=\mathbbm{1}_{|\cI_>|}$ minimizes the sum, since it is a feasible point.


In summary, the expected value of the normalized MTF can be calculated as
\begin{equation}
    \E[H_n] = \begin{cases}
        \frac{\sqrt{C(N,n)}}{\lfloor N/2 \rfloor2^{|\cI_>|}} \sum_{u \in \cube} \sqrt{1+\frac{2a^\top u}{C(N,n)}} & n = 1, \ldots, \lfloor \frac{N}{2} \rfloor - 2\\
        \frac{1}{\lfloor N/2 \rfloor} & n = \lfloor \frac{N}{2} \rfloor - 1
    \end{cases}.
\end{equation}
It can furthermore be lower-bounded by
\begin{equation}
    \E[H_n] \geq \begin{cases}
        \frac{\sqrt{C(N,n)}}{\lfloor N/2 \rfloor2^{|\cI_>|}} \sum_{u \in \cube} \sqrt{\left(1+\frac{2\mathbbm{1}^\top u}{C(N,n)}\right)_+} & n = 1, \ldots, \lfloor \frac{N}{2} \rfloor - 2\\
        \frac{1}{\lfloor N/2 \rfloor} & n = \lfloor \frac{N}{2} \rfloor - 1
    \end{cases},
\end{equation}
where $(x)_+ = x\ind{x \geq 0}$.
\end{proof}

\begin{proof}[Proof of Theorem~\ref{thm:binary-squared-mtf}]
    Recall our setting: we want to find the expected square magnitude of the autocorrelation of 
    \begin{equation}
        \tilde{P}_j = A_je^{i\phi_j}e^{i\pi B_j}.
    \end{equation}
    Since the autocorrelation sequence is $N$ periodic and symmetric about $n=0$, we will solve it only for $n = 0, \dots, \lfloor \frac{N}{2} \rfloor - 1$. For convenience, we will omit the square of the normalization factor $\frac{1}{\lfloor N/2\rfloor}$ till the end. Writing out the unnormalizaed autocorrelation sequence for these values gives
    \begin{equation}
        \bar{H}_n = \left | \sum_{j=0}^{N-1} \tilde{P}_j\tilde{P}^\ast_{j-n} \right | =  \left | \sum_{j=0}^{N-1} A_n^2 e^{i(\phi_j - \phi_{j-n})}e^{i\pi B_j}e^{-i\pi B_{j-n}} \right | = \left | \sum_{j=n}^{\lfloor N/2\rfloor - 1} e^{i(\phi_j - \phi_{j-n})} R_j R_{j-n} \right |,
    \end{equation}
    where the $R_j = e^{i\pi B_j}$ are i.i.d. as per Lemma~\ref{lem:rademacher-phase}.
    Next, we take the expected square of this expression to get
    \begin{equation}
        \mathbb{E}\left|\sum_{j=n}^{\lfloor N/2\rfloor - 1} e^{i(\phi_j - \phi_{j-n})} R_j R_{j-n}\right|^2 = \sum_{j=n}^{\lfloor N/2\rfloor - 1} \mathbb{E}[(R_jR_{j-n})^2] + \sum_{j=n}^{\lfloor N/2\rfloor - 1}\sum_{k \neq j} \mathbb{E}[R_jR_{j-n}R_kR_{k-n}]e^{i(\phi_j - \phi_{j-n} - \phi_k + \phi_{k-n})},
    \end{equation}
    
    where we utilize the expanded form of a product of sums and apply linearity of expectation. Since each $R_j$ can only be $1$ or $-1$, $R_jR_{j-n}^2 = 1$ deterministically. Thus, $\mathbb{E}[(R_jR_{j-n})^2] = \mathbb{E}[1] = 1$, and so the first term simplifies to 
    \begin{equation}
        \sum_{j=n}^{\lfloor N/2\rfloor - 1} \mathbb{E}[(R_jR_{j-n})^2] = \lfloor N/2\rfloor - n.
    \end{equation}

    Now we will decompose the second term (the cross terms) into a variety of cases based on the sum indices. Doing so will allow us to compute all of the expectations. Note that, the normalized autocorrelation sequence is normalized by it's value at $n=0$ such that $\tilde{H}_0 = 1$. Consequently, we will only consider the following cases when $n>0$ and remedy this in the final expression. 

    \noindent \emph{Case 1: $k\neq j \pm n$}.
    In this case, none of the indices in the expectation overlap, which allows us to invoke independence and separate the terms. The expectation of $R_j$, from lemma~\ref{lem:rademacher-phase}, is $\mathbb{E}[R_j] = (1-p) - p = (1-2p)$. Combining these facts gives
    \begin{equation}
        \mathbb{E}[R_jR_{j-n}R_kR_{k-n}] = \mathbb{E}[R_j]\mathbb{E}[R_{j-n}]\mathbb{E}[R_k]\mathbb{E}[R_{k-n}] = (1-2p)^4.
    \end{equation}
    Thus, the full expression of the second term for terms under this case becomes
    \begin{equation}
          (1-2p)^4\sum_{j=n}^{\lfloor N/2\rfloor -1}\sum_{\substack{k \neq j \\ k \neq j \pm n}} e^{i(\phi_j - \phi_{j-n} - \phi_k + \phi_{k-n})}.
    \end{equation}

    \noindent \emph{Case 2: $k = j + n$}.
    Considering only terms for which $k = j + n$, yields the following expression

    \begin{equation}
        \sum_{j=n}^{\lfloor N/2\rfloor - n - 1} \mathbb{E}[R_j^2]\mathbb{E}[R_{j-n}R_{j+n}]e^{i(2\phi_j- \phi_{j-n} - \phi_{j+n})} = (1-2p)^2\sum_{j=n}^{\lfloor N/2\rfloor - n -1} e^{i(2\phi_j- \phi_{j-n} - \phi_{j+n})}
    \end{equation}

    \noindent \emph{Case 3: $k = j - n$}.
    Similarly, considering only terms for which $k = j - n$, yields the following expression

    \begin{equation}
        \sum_{j=2n}^{\lfloor N/2\rfloor - 1} \mathbb{E}[R_{j-n}^2]\mathbb{E}[R_{j}R_{j-2n}]e^{i(\phi_j- 2\phi_{j-n} - \phi_{j-2n})} = (1-2p)^2\sum_{j=2n}^{\lfloor N/2\rfloor - 1} e^{i(\phi_j- 2\phi_{j-n} - \phi_{j-2n})}
    \end{equation}

Having covered all cases, we can assemble ours result along with the squared normalization factor $\frac{1}{\lfloor N/2\rfloor ^2}$ to arrive at the final expression.
\end{proof}

\section{Technical lemmas}
\label{appendix:lemmas}
\begin{lemma}
    \label{lem:pupil-to-MTF}
    In the setting of~\eqref{eq:pupil-masked}, we have that
    \begin{equation}
        H_n = \frac{1}{\lfloor N/2 \rfloor}\left| \sum\limits_{j=n}^{N-1} e^{i(\phi_j - \phi_{j-n} + W_j - W_{j-n})}\right|.
    \end{equation}
\end{lemma}
\begin{proof}
    This can be verified by directly plugging into~\eqref{eq:mtf-def}.
\end{proof}

\begin{lemma}[Translation invariance of uniform phasors.]
    \label{lem:translation-invariance-uniform}
    Let $U \sim \mathrm{Unif}(0, 2\pi)$ and consider any $\phi \in \R$. Then
    \begin{equation}
        e^{i(\phi + U)} \eqd e^{iU}.
    \end{equation}
\end{lemma}
\begin{proof}
    Define $U' \sim \mathrm{Unif}(\phi, 2\pi + \phi)$.
    Note that $U'\mod2\pi \eqd U$.
    Thus $e^{i(\phi + U)} \eqd e^{i(U'\mod(2\pi))} \eqd e^{iU}$.
\end{proof}

\begin{lemma}
    \label{lem:rademacher-phase}
    Let $B$ be a Bernoulli random variable with probability parameter $p$. Then random variable $R \overset{\Delta}{=} e^{i\pi B} = e^{-i\pi B} $ is specified by
    \begin{equation}
        R = \begin{cases}
            1 & \text{w.p.} \ 1-p \\
            -1 & \text{w.p.} \ p 
        \end{cases}
    \end{equation}
    Note that $R$ here is a Rademacher random variable when $p = 0.5$.
\end{lemma}
\begin{proof}
Since $B$ is Bernoulli, it can only be either $1$ or $0$. In the case that it is $1$, $R = e^{i\pi B} = e^{i\pi} = -1$, which happens with probability $p$. In the case that it is $0$, $R = e^{i\pi B} = e^{0} = 1$, which happens with probability $1-p$.
\end{proof}

\begin{lemma}
    \label{lem:prod-rad-unif}
    Let $R_0, \dots, R_{M}$ be jointly independent Rademacher random variables. Then for $n>0$ 
    $\begin{bmatrix} R_{n}R_{0} \\ R_{n+1}R_{1} \\ \vdots \\ R_{M}R_{M-n}   \end{bmatrix}$ is uniformly distributed on the hypercube.
\end{lemma}

\begin{proof}
    
    For convenience, set $M = \lfloor N/2\rfloor - 1$. Let $U = \begin{bmatrix} R_{n}R_{0} \\ R_{n+1}R_{1} \\ \vdots \\ R_{M}R_{M-n}   \end{bmatrix}$. We will proceed by induction on $M$ for $n = 1, \dots, M-1$. The case of $n=m$ follows since $U$ will only have a single element $R_MR_0$ which is uniform on the hypercube because $\mathbb{P}(R_MR_0 = \pm 1) = 0.5$. For the base case, let $M=2$, then the only relevant case is $n=1$ for which $U = \begin{bmatrix} R_{1}R_{0} \\ R_{2}R_{1} \end{bmatrix}$. Note that 
    \begin{align}
    \mathbb{P}(U = u) = \mathbb{P}(U_{2} = u_2 \mid U_{1} = u_1)\mathbb{P}(U_{1} = u_1) &= \mathbb{P}(R_{2}R_{1} = u_2 \mid R_{1}R_{0} = u_1)\mathbb{P}(R_{1}R_{0} = u_1) \\ &= \mathbb{P}(R_{2}R_{1} = u_2 \mid R_{1}R_{0} = u_1)0.5.
    \end{align}
    Now $R_{2}R_{1}$ will only take values $1$ or $-1$ and it will do so with equal probability regardless of the value of $R_{2}R_{1}$ because $R_2$ is unaffected. Thus $\mathbb{P}(R_{2}R_{1} = \pm 1 \mid R_{1}R_{0} = u_1)0.5 = 0.5^2$ and $U$ is uniform on the hypercube.
    
    Now let's assume that for an arbitrary $M$ and $n \in \{1,...M-1\}$ we have 
    \begin{align}
        \mathbb{P}(U_{1:M-n-1} = u_{1:M-n-1}) = \mathbb{P}(R_nR_0 = u_1)\mathbb{P}(R_n+1R_1 = u_2) \dots \mathbb{P}(R_{M-1}R_{M-n-1} = U_{M-n-1}) = 0.5^{M-n-1}.
    \end{align}
    Then, using the law of total probability, we can write 
    \begin{align}
        \mathbb{P}(U = u) &= \mathbb{P}(U_{M-n} = u_{M-n}  \mid U_{1:M-n-1} = u_{1:M-n-1})\mathbb{P}(U_{1:M-n-1} = u_{1:M-n-1}) \\ &= \mathbb{P}(R_{M}R_{M-n} = u_{M-n} \mid U_{1:M-n-1} = u_{1:M-n-1})0.5^{M-n-1}. 
    \end{align}
    Now by a similar logic as the base case, fixing all of the elements of $U_{1:M-n-1}$ will not effect the outcome of $R_M$. Thus we have 
    \begin{equation}
        \mathbb{P}(R_{M}R_{M-n} = \pm 1 \mid U_{1:M-n-1} = u_{1:M-n-1}) = 0.5 = \mathbb{P}(R_{M}R_{M-n} = \pm 1),
    \end{equation}
    and so 
    \begin{equation}
        \mathbb{P}(R_{M}R_{M-n} = u_{M-n} \mid U_{1:M-n-1} = u_{1:M-n-1})0.5^{M-n-1} = 0.5^{M-n}
    \end{equation}
    which means that that $U$ is uniform on the hypercube.
\end{proof}

\end{document}